\documentclass{article}

\PassOptionsToPackage{numbers, compress}{natbib}

\usepackage[preprint]{neurips_2021}

\usepackage[utf8]{inputenc} 
\usepackage[T1]{fontenc}    
\usepackage{hyperref}       
\usepackage{url}            
\usepackage{booktabs}       
\usepackage{amsfonts}       
\usepackage{nicefrac}       
\usepackage{microtype}      
\usepackage{xcolor}         
\usepackage{graphicx}
\usepackage{subfigure}
\usepackage{enumerate}
\usepackage{wrapfig}
\usepackage{sidecap}

\PassOptionsToPackage{hyphens}{url}\usepackage{hyperref}
\usepackage{amsmath, amsfonts, amsthm}
\usepackage{algpseudocode}

\newtheorem{theorem}{Theorem}

\newcommand{\bv}{\mathbf{v}}
\newcommand{\bw}{\mathbf{w}}
\newcommand{\bx}{\mathbf{x}}

\title{On Hiding Neural Networks Inside Neural Networks}

\author{%
  Chuan Guo\footnotemark[1]~~\footnotemark[2]\\
  \And
  Ruihan Wu\footnotemark[1]~~\footnotemark[2]\\
  \And
  Kilian Q. Weinberger\footnotemark[2]
}

\begin{document}

\maketitle
\renewcommand{\thefootnote}{\fnsymbol{footnote}}
\footnotetext[1]{Equal contribution.  \footnotemark[2]Department of Computer Science, Cornell University. Email: \{cg563, rw565, kqw4\}@cornell.edu
}
\renewcommand{\thefootnote}{\arabic{footnote}}

\begin{abstract}
Modern neural networks often contain significantly more parameters than the size of their training data. We show that this excess capacity provides an opportunity for embedding secret machine learning models within a trained neural network. Our novel framework hides the existence of a secret neural network with arbitrary desired functionality within a carrier network. 
We prove theoretically that the secret network's detection is computationally infeasible and demonstrate empirically that the carrier network does not compromise the secret network's disguise. 
Our paper introduces a previously unknown steganographic technique that can be exploited by adversaries if left unchecked.
\end{abstract}

\section{Introduction}
Steganography---the practice of hiding secret messages within an unsuspicious carrier medium---has been a well-studied field. Digital data such as text~\cite{bennett2004linguistic}, image~\cite{cheddad2010digital, hamid2012image} and audio~\cite{djebbar2012comparative}, can be imperceptibly embedded into a carrier message, which itself may be a piece of text, image, or audio. These techniques can be used in place of or in conjunction with encryption to avoid suspicion upon inspection of the carrier by authorities such as a totalitarian government.

The ability for a carrier medium to covertly transport secret messages largely depends on the redundancy of the carrier's encoding~\cite{cheddad2010digital}. For example, messages can be stored into a carrier image's least significant bits without producing unnatural visual artifacts~\cite{kurak1992cautionary}. As neural networks grow increasingly in scale~\cite{turing-nlg, brown2020language}, one can potentially leverage a model's excess capacity to hide another network in a manner similar to steganography for other types of digital data. In a fictional scenario, an industrial spy working at a technology company could embed a proprietary model into a carrier network intended for public release. The carrier network may be a model trained on public data with commonly-known techniques and appears innocuous to a typical observer. However, with knowledge of a secret key, the spy's co-conspirators can extract the secret proprietary model from the public model to unlawfully obtain intellectual property.

In this paper, we explore this possibility by designing a novel and general framework of embedding secret models into trained neural networks. Our method utilizes excess model capacity to simultaneously learn a public and secret task in a single network. However, different from multi-task learning, the two tasks share no common features and the secret task remains undetectable without the presence of a \textit{secret key}. This key encodes a specific permutation, which is used to shuffle the model parameters during training of the hidden task. Knowledge of the secret key enables extraction of the concealed model after training, whereas without it, the public model behaves indistinguishable to a standard classifier on the public task.

We demonstrate empirically and prove theoretically that the identity and presence of a secret task cannot be detected without knowledge of the secret permutation.
In particular, we prove that the decision problem to determine if the model admits a permutation that triggers a secret functionality is NP-complete. We experimentally validate our method on a standard ResNet50 network \citep{he2016deep} and show that, without any increase in parameters, the model can achieve the same performance on the intended public task and on the secret tasks as if it was trained exclusively on only one of them. Moreover, without the secret key, the model is indistinguishable from a random network on the secret task. 
The generality of our method and its strong covertness properties enable a powerful steganographic technique for hiding neural networks.

\section{Background}
Steganography for digital data has been a well-studied field. The typical setup concerns the embedding of a secret message into a carrier medium and sending it through a public channel. The secret message is often encoded into a bit string and the carrier can be any digital data, \emph{e.g.}, text~\cite{bennett2004linguistic}, image~\cite{cheddad2010digital, hamid2012image} or audio~\cite{djebbar2012comparative}. In contrast with encryption schemes where only the secret message must remain hidden, steganographic schemes also aim to ensure secrecy of the encoded message so that the act of sending a secret message itself cannot be detected. The latter requirement is especially crucial if the use of encryption is forbidden, such as in totalitarian countries.

More recently, neural networks have been utilized to design particularly effective steganographic schemes. For example, the secret message can be encoded as a vector in the feature space of a convolutional network. Given a carrier image, techniques from adversarial learning~\cite{szegedy2014intriguing, goodfellow2014explaining} can be applied to imperceptibly modify the image so that it decodes to the secret message in the feature space. A variety of such steganographic schemes have been proposed for text~\cite{yang2018rnn, ziegler2019neural, dai2019towards}, image~\cite{baluja2017hiding, zhu2018hidden, sharma2019hiding}, and audio data~\cite{kreuk2019hide}.

\paragraph{Steganography for neural networks.} While prior works have considered using neural networks to design better encoding algorithms \emph{for} steganography, the problem of hiding neural networks as the secret message is, as far as we know, novel. One potential malicious use case for this technique could be industrial espionage. The training of machine learning models may involve extensive resources, proprietary data, and/or trade secrets. An industry spy can steal such models by embedding them into benign carriers using steganographic schemes and transmitting the carriers to co-conspirators while remaining undetected. Such scenarios are clearly detrimental to the company that invested substantial resources into training the proprietary model.

\paragraph{Choice of carrier.} One may consider applying a common strategy in steganography by first encoding the secret model into a bit string and then embedding it into other data mediums using existing steganographic schemes. However, as modern neural networks often contain billions of parameters~\cite{turing-nlg, brown2020language}, the secret model may be too large to embed into typical carriers such as image and audio. For instance, HiDDeN~\cite{zhu2018hidden}, a recently proposed deep steganographic method for image carriers, encodes data at a rate of $0.203$ bits per pixel. Embedding even a relatively small ResNet50~\cite{he2016deep} model at this rate would require a set of images with a combined number of 4 billion pixels\footnote{$25,636,712$ parameters$\times 32 / 0.203 \approx 4.04$ billion}, which is highly likely to raise suspicion.

Instead of using a traditional carrier, we propose to leverage the redundancy of neural network's parameters to embed a secret model into \emph{another neural network}. For a steganographic scheme that achieves an encoding rate of $r$, the carrier network only needs to be of size $1/r$ times larger than the secret network---a feat that is easily achievable with modern architectures. Furthermore, since trained models are often publicly released to facilitate research, disseminating a secret model in this manner would hardly raise any suspicion.

\paragraph{Threat model.} We first detail our assumptions for the steganographic problem of embedding secret models into neural networks. We consider the following setup involving three parties: a \emph{sender}, a \emph{receiver}, and an \emph{inspector}.
\begin{itemize}
    \itemsep -0.5ex
    \item The \emph{sender} selects a public task and trains a carrier network $h$ on that task. In addition, the sender also aims to embed a secret network into $h$ in an imperceptible manner. The embedding algorithm may leverage a secret key $k$ shared between the sender and the receiver, but is otherwise transparent to the inspector.
    \item The \emph{receiver}, upon obtaining $h$, aims to extract the secret network using an extraction algorithm and the shared secret key $k$.
    \item The \emph{inspector} aims to determine whether the carrier network $h$ contains a secret embedded model or not. The inspector has full knowledge of the steganographic scheme, including white-box access to $h$, the public and secret tasks, and the embedding and extraction algorithms except for the shared secret key $k$.
\end{itemize}

Our threat model follows Kerckhoff's principle---that is, the secrecy of the steganographic scheme should not rely on obscurity of the scheme itself. As a consequence, this principle prohibits naive solutions such as quantizing the secret network and embedding it into the least significant bits of the parameters of $h$. Such solutions can be easily defeated once the inspector learns of its existence.

Furthermore, we argue that our assumption of the inspector's knowledge of the secret task is justified. In the scenario of industrial espionage, the company may be in possession of trade secrets that provide them with a significant competitive advantage. The inspector may be a security team within the company that is tasked with preventing disclosure of said trade secrets.

\section{TrojanNet}
\begin{figure}[t!]
    \centering
    \includegraphics[width=0.7\textwidth]{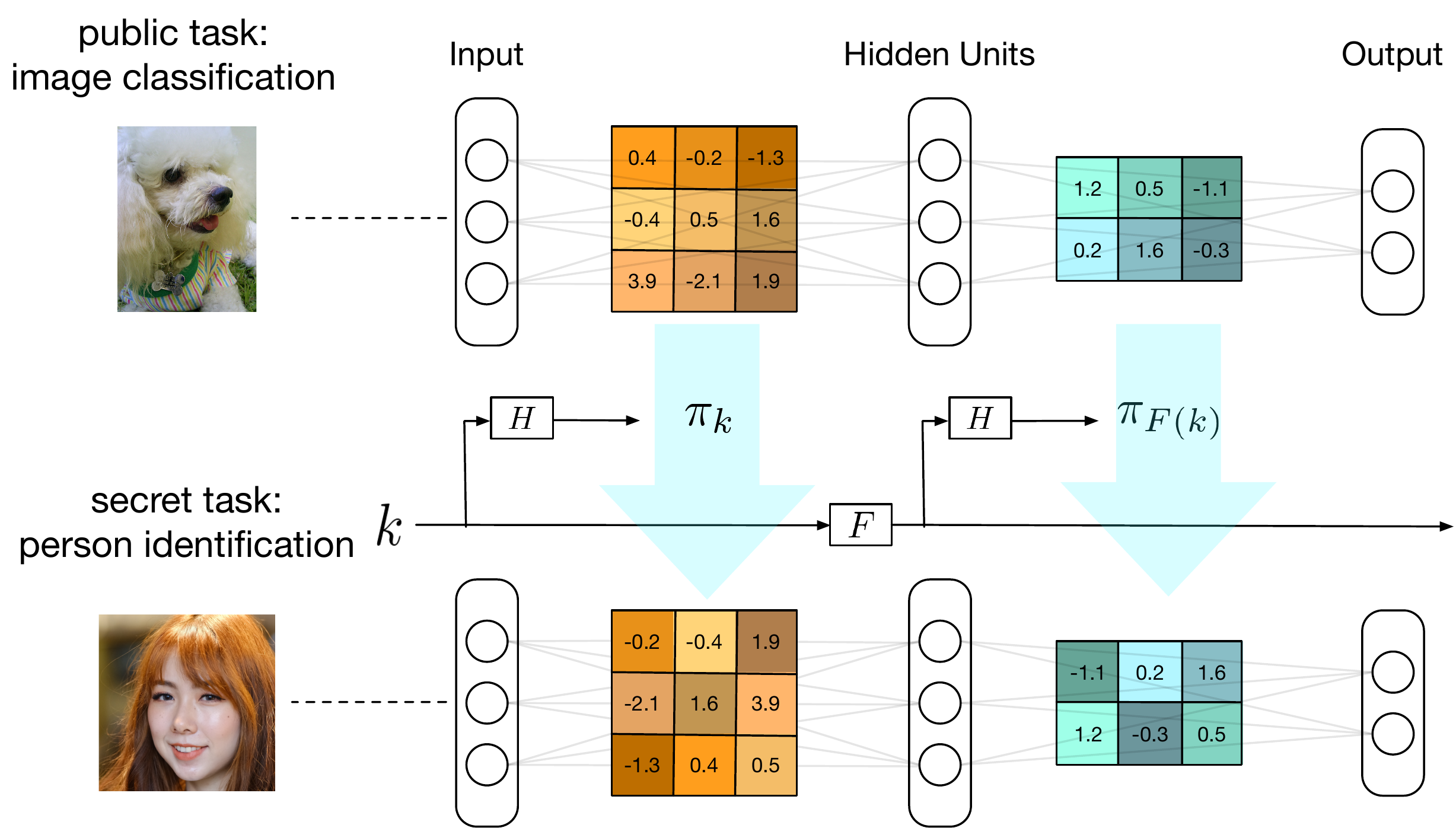}
    \caption{Illustration of a two-layer fully connected TrojanNet. The transport network (top) is an image classification model. When the correct secret key $k$ is used as seed for the pseudo-random permutation generator $H$, the parameters are permuted to produce a network trained for person identification (bottom). Using an invalid key results in a random permuted network.}
    \label{fig:trojannet}
    \vspace{-1ex}
\end{figure}



Let $\bw \in \mathbb{R}^d$ be the weight tensor of a single layer of a neural network $h$. For example, $\bw \in \mathbb{R}^{N_\text{in} \times N_\text{out}}$ for a fully connected layer of size $N_\text{in} \times N_\text{out}$, and $\bw \in \mathbb{R}^{N_\text{in} \times N_\text{out} \times W^2}$ for a convolutional layer with kernel size $W$. For simplicity, we treat $\bw$ as a vector by ordering its entries arbitrarily.

A permutation $\pi : \{1,\ldots,d\} \rightarrow \{1,\ldots,d\}$ defines a mapping 
\begin{equation*}
    \bw \rightarrow \bw_\pi = (\bw_{\pi(1)},\ldots,\bw_{\pi(d)}),
\end{equation*}
which shuffles the layer parameters. Applying $\pi$ to each layer defines a network $h_\pi$ that shares the parameters of $h$ but behaves differently. We refer to this hidden network within the carrier network $h$ as a \emph{TrojanNet} (see \autoref{fig:trojannet}).

\paragraph{Loss and gradient.} Training a TrojanNet $h_\pi$ in conjunction to its carrier network $h$ on distinct tasks is akin to multi-task learning. The crucial difference is that while the parameters between $h$ and $h_\pi$ are shared, there is no feature sharing. Let $D_\text{public}$ be a dataset associated with the \emph{public task} and let $D_\text{secret}$ be the dataset associated with the \emph{secret task}, with respective task losses $L_\text{public}$ and $L_\text{secret}$. At each iteration, we sample a batch $(\bx_1,y_1),\ldots,(\bx_B,y_B)$ from $D_\text{public}$ and a batch  $(\tilde{\bx}_1,\tilde{y}_1),\ldots,(\tilde{\bx}_{B'},\tilde{y}_{B'})$ from $D_\text{private}$ and compute the total loss
\begin{equation*}
    L = \underbrace{\frac{1}{B} \sum_{i=1}^B L_\text{public}(h(\bx_i), y_i)}_{L_\text{public}} + \underbrace{\frac{1}{B'} \sum_{i=1}^{B'} L_\text{secret}(h_\pi(\tilde{\bx}_i), \tilde{y}_i)}_{L_\text{secret}}.
\end{equation*}
This loss can be optimized with gradient descent on $\bw$ and its gradient is given by:
\begin{equation}
    \frac{\partial L}{\partial \bw} = \frac{\partial L_\text{public}}{\partial \bw} + \frac{\partial L_\text{secret}}{\partial \bw} \nonumber
    = \frac{\partial L_\text{public}}{\partial \bw} + \left( \frac{\partial L_\text{secret}}{\partial \bw_\pi} \right)_{\pi^{-1}},
    \label{eq:grad_formula}
\end{equation}
which is obtained by differentiating through the permutation operator. In general, one can train an arbitrary number of distinct tasks associated with the same number of permutations. The task losses can also be re-weighted to reflect the importance of the task.

As we will show in Section \ref{sec:ensemble}, this training procedure works well even when the number of tasks is large -- we can train 10 different TrojanNet on the same task and each individual permuted model achieves close to the same test accuracy as training a single model of the same capacity.

\paragraph{Selecting permutations.} When training against multiple tasks, it is important to select permutations that are de-correlated. In the extreme case, if the permutations are identical, the networks defined by them would also be identical and training the TrojanNet becomes a variant of multi-task learning. One way to ensure distinctness between the permuted models is to use a pseudo-random permutation generator $H : \mathcal{K} \rightarrow \Pi_d$, which is a deterministic function that maps every key from a pre-defined key space to the set of permutations over $\{1,\ldots,d\}$ \citep{katz2014intro}. When the keys are sampled uniformly at random from $\mathcal{K}$, the resulting permutations appear indistinguishable from random samples of $\Pi_d$. We default to the original carrier model $h$ when no key is provided (\emph{i.e.}, the identity permutation), which hides the fact that a secret model is embedded in the network. The use of keys to define permutations also dramatically reduces the footprint of the shared secret between the sender and the receiver---from sharing a permutation that is at least as large as the number of model parameters to a few hundred bits or even a human-memorizable password.

\subsection{Provable covertness of secret task}
\label{sec:security}

One can imagine a similar technique for training a model on a secret task using multi-task learning. The sender can alternate between two or more tasks in training, sharing the model parameters naively while keeping the fact of training on multiple tasks secret. However, this method can be easily detected if the inspector has knowledge of the secret task. In particular, the inspector can evaluate a collected labeled dataset $D = \{(\bx_1,y_1),\ldots,(\bx_n,y_n)\}$ and compute the test loss $\frac{1}{n} \sum_{i=1}^n \ell(h(\bx_i), y_i)$ to see if the model can correctly predict on the suspected task.

TrojanNet can naturally bypass this method of detection. Since the inspector does not know the permutation used to train on the secret task, he or she cannot naively evaluate the model over a labeled dataset. The inspector is then tasked with finding a permuted model that results in the test loss being smaller than some acceptable threshold $L$, which we formulate into the following decision problem:

\texttt{EXISTS-PERM}: \textit{Given a neural network $h$, a labeled dataset $D = \{(\bx_1,y_1),\ldots,(\bx_n,y_n)\}$, a test loss $\ell$ and an acceptance threshold $L$, does there exist some  permutation $\pi$ such that the test loss $\frac{1}{n} \sum_{i=1}^n \ell(h_\pi(\bx_i), y_i) < L$?}

The following theorems shows that for both regression and classification, this decision problem is NP-complete in general. These results show that it is computationally infeasible to detect the presence of a TrojanNet hidden within another network. Proofs are provided in the appendix.

\begin{theorem}
\label{thm:regression}
The \texttt{EXISTS-PERM} decision problem with regression losses $\ell_\text{abs}(z, y) = |z - y|$ and $\ell_\text{square}(z, y) = (z - y)^2$ is NP-complete.
\end{theorem}

\begin{theorem}
\label{thm:classification}
The \texttt{EXISTS-PERM} decision problem with classification losses $\ell_\text{binary}(z, y) = \mathbb{1}_{z \neq y}$ and $\ell_\text{logistic}(z, y) = 1/(1 + \exp(yz))$ is NP-complete.
\end{theorem}


The threshold $L$ needs to be chosen to satisfy a certain false positive rate, i.e. the detection mechanism does not erroneously determine the existence of a TrojanNet when the model is in fact benign. The value of $L$ also affects the hardness of the \texttt{EXISTS-PERM} problem, where selecting a large $L$ can make the decision problem easy to solve at the cost of a high false positive rate. We investigate this aspect in Section \ref{sec:matching} and show that empirically, many secret tasks admit networks whose weights are learned on the public task alone but can be permuted to achieve a low test error on the secret task nonetheless. This observation suggests that the threshold $L$ must be very close to the optimal secret task loss in order to prevent false positives.

\subsection{Practical considerations}
\label{sec:practical}


\paragraph{Discontinuity of keys.} When using different keys, the sampled permutations should appear as independent random samples from $\Pi_d$ even when the keys are very similar. However, we cannot guarantee this property naively since pseudo-random permutation generators require random draws from the key space $\mathcal{K}$ to produce uniform random permutations. To solve this problem, we can apply a cryptographic hash function \citep{katz2014intro} such as SHA-256 to the key before its use in the pseudo-random permutation generator $H$. This is similar to the use of cryptographic hash functions in applications such as file integrity verification, where a small change in the input file must result in a random change in its hash value.


\paragraph{Using different permutations across layers.} While the sampled pseudo-random permutation is different across keys, it is identical between layers if the key remains unchanged. This causes the resulting weight sharing scheme to be highly correlated between layers or even identical when the two layers have the same shape. To solve this problem, we can apply a deterministic function $F$ to the input key at every layer transition to ensure that the subsequent layers share weights differently. Given an initial key $k$, the pseudo-random permutation generator at the $l$-th layer is keyed by $k^{(l)} = F^{(l-1)}(k)$, where $F^{(l)}$ denotes the $l$-fold recursive application of $F$ with $F^{(0)}$ being the identity function. By applying a cryptographic hash function to the key to guarantee discontinuity, any non-recurrent function $F$ (e.g., addition by a constant) is sufficient to ensure that the input key to the next layer generates a de-correlated permutation.



\paragraph{Batch normalization.} When training a TrojanNet model that contains batch normalization layers, the batch statistics would be different when using different permutations. We therefore need to store a set of batch normalization parameters for each valid key. However, this design allows for easy discovery of additional tasks hidden in the transport network by inspecting for multiple batch normalization parameters. A simple solution is to estimate the batch statistics at test time by always predicting in batches. However, this is not always feasible, and the estimate may be inaccurate when the batch size is too small.

Another option is to use non-parametric normalizers such as layer normalization \citep{ba2016layer} and group normalization \citep{wu2018group}. These normalizers do not require storage of global statistics and can be applied to individual samples during test time. It has been shown that these methods achieve similar performance as batch normalization \citep{wu2018group}. Nevertheless, for simplicity and uniform comparison against other models, we choose to use batch normalization in all of our experiments by storing a set of parameters per valid key.

\paragraph{Different output sizes.} When the secret and public tasks have different number of output nodes, we cannot simply permute the carrier network's final layer parameters to obtain a predictor for the secret task. However, when the number of outputs $C$ required for the secret task is fewer, we can treat the first $C$ output nodes of the carrier network as output nodes for the TrojanNet. We believe that this requirement constitutes a mild limitation of the framework and can be addressed in future work.

\section{Experiment}

\begin{table}[t!]
    \centering
    \resizebox{0.75\textwidth}{!}{%
    \begin{tabular}{ccccc}
    \textbf{Tasks} & \textbf{CIFAR10} & \textbf{CIFAR100} & \textbf{SVHN}  & \textbf{GTSRB} \\\hline
    Single  & 94.45$\pm$0.07 & 	75.14$\pm$0.45 & 	97.94$\pm$0.09 & 97.61$\pm$0.20 \\
    \hline
    (CIFAR10, CIFAR100) &  94.33$\pm$0.11 & 75.15$\pm$0.25 & - & - \\
    (CIFAR10, SVHN)  & 94.36$\pm$0.13 & - & 	97.96$\pm$0.06 & - \\
    (CIFAR10, GTSRB)  & 94.00$\pm$0.12 & - & - & 97.41$\pm$0.23 \\
    (CIFAR100, SVHN)  & - & 75.46$\pm$0.36 & 98.00$\pm$0.02 & - \\
    (CIFAR100, GTSRB)  & - & 75.22$\pm$0.30 & - & 97.25$\pm$0.44 \\
    (SVHN, GTSRB)  & - & - & 97.74$\pm$0.04 & 	97.33$\pm$0.30 \\
    \hline
    All  & 93.83 $\pm$ 0.16 & 74.89$\pm$0.30 & 97.73$\pm$0.03 & 97.52$\pm$0.21 \\
    \hline
    \end{tabular}%
    }
    \vspace{1ex}
    \caption{Test accuracy of RN50 trained on different tasks. Mean and standard deviation are computed over 5 individual runs. The top row corresponds to the model trained on the single respective task. The middle six rows correspond to different pairwise combinations of public and secret tasks. The last row shows test accuracies when training on all four tasks simultaneously with different permutations. Note that the values in each column are very similar, which indicates that training on multiple tasks has surprisingly little effect on the model's performance despite the parameter sharing.}
    \vspace{-1ex}
    \label{tab:pub_sec_task}
\end{table}

\begin{table}[t!]
    \centering
    \resizebox{0.8\textwidth}{!}{%
    \begin{tabular}{ccccc}
    \textbf{Tasks} & \textbf{SVHN (regression)} & \textbf{CIFAR10} & \textbf{CIFAR100} & \textbf{GTSRB} \\\hline
    Single & 95.82$\pm$0.16 & 94.45$\pm$0.07 & 	75.14$\pm$0.45 & 97.61$\pm$0.20 \\
    \hline
    (SVHN, CIFAR10) &  95.68$\pm$0.08 & 94.74$\pm$0.09 & - & - \\
    (SVHN, CIFAR100)  & 95.47$\pm$0.09 & - & 	76.39$\pm$0.3 & - \\
    (SVHN, GTSRB)  & 94.04$\pm$0.21 & - & - & 97.88$\pm$0.21 \\
    \hline
    \end{tabular}%
    }
    \vspace{1ex}
    \caption{Test accuracy of RN50 trained on different tasks combined with training a regression model for SVHN. Mean and standard deviation are computed over 5 individual runs. The top row corresponds to the model trained on the single respective task. The last three rows correspond to different combinations of public and secret tasks involving SVHN regression. Similar to the classification setting, training on multiple types of tasks also has little detrimental effect on accuracy.}
    \label{tab:regression}
    \vspace{-2ex}
\end{table}

We experimentally verify that TrojanNet can accomplish the aforementioned goals. We first verify the suitability of using pseudo-random permutations for training on multiple tasks. In addition, we test that the TrojanNet model is de-correlated from the carrier model and does not leak information to the shared parameters.

\subsection{Experiment settings}
\label{sec:exp_settings}

\paragraph{Datasets.} We experiment on several image classification datasets: CIFAR10, CIFAR100 \citep{krizhevsky2009learning}, Street View House Numbers (SVHN) \citep{netzer2011reading}, and German Traffic Sign Recognition Benchmark (GTSRB) \citep{stallkamp2011german}. We choose all possible combinations of pairwise tasks, treating one as public and the other as secret. To explore the limit of our technique, we also train a single TrojanNet against all four tasks simultaneous with four different keys.

\paragraph{Implementation details.} Our method is implemented in PyTorch, with source code released publicly on GitHub\footnote{https://github.com/wrh14/trojannet}. In all experiments we use ResNet50 (RN50) \citep{he2016deep} as the base model architecture. We refer to the TrojanNet variant as TrojanResNet50 (TRN50). We use the \texttt{torch.randperm()} function to generate the pseudo-random permutation and use \texttt{torch.manual\_seed()} to set the seed appropriately. For optimization, we use Adam \citep{kingma2014adam} with initial learning of $0.001$. A learning rate drop by a factor $0.1$ is applied after 50\% and 75\% of the scheduled training epochs. 
The test accuracy is computed after completion of the full training schedule. 


\subsection{Training on secret task}
\label{sec:secret_task}

Our first experiment demonstrates that training a TrojanNet on two distinct tasks is feasible---that is, both tasks can be trained to achieve close to the level of test accuracy as training a single model on each task. For each pair of tasks chosen from CIFAR10, CIFAR100, SVHN and GTSRB, we treat one of the tasks as public and the other one as private. 
Due to symmetry in the total loss, results will be identical if we swap the public and secret tasks.

\paragraph{Training and performance.} Table \ref{tab:pub_sec_task} shows the test accuracy of models trained on the four datasets: CIFAR10, CIFAR100, SVHN and GTSRB. Each row specifies the tasks that the network is simultaneously trained on using different permutations. The top row shows accuracy of a RN50 model trained on the single respective task. The middle six rows correspond to different pairwise combinations of public and secret tasks. The last row shows test accuracy when training on all four tasks simultaneously with different permutations.

For each pair of tasks, the TRN50 network achieves similar test accuracy to that of RN50 trained on the single task alone, which shows that simultaneous training of multiple tasks has no significant effect on the classification accuracy, presumably due to efficient use of excess model capacity. Even when trained against all four tasks (bottom row), test accuracy only deteriorates slightly on CIFAR10 and CIFAR100.
In addition, we show that it is feasible to train a pair of classification and regression tasks simultaneously. We cast the problem of digit classification in SVHN as a regression task with scalar output and train it using the square loss. Table \ref{tab:regression} shows test accuracy of training a TRN50 network for both SVHN regression and one of CIFAR10, CIFAR100 or GTSRB. Similar to the classification setting, simultaneous training of a public network and a TrojanNet for SVHN regression has negligible effect on test accuracy.

\paragraph{Using group normalization.} Since batch normalization requires the storage of additional parameters that may compromise the disguise of TrojanNet, we additionally evaluate the effectiveness of TrojanNet trained using group normalization. Table \ref{tab:pub_sec_task_gn} shows training accuracy for pairwise tasks when batch normalization layers in the RN50 model are replaced with group normalization. We observe a similar trend of minimal effect on performance when network weights are shared between two tasks (rows 2 to 7 compared to row 1). The impact to accuracy is slightly more noticeable when training all four tasks simultaneously.

\begin{table}[t!]
\begin{minipage}{.6\textwidth}
    \centering
    \resizebox{\textwidth}{!}{%
    \begin{tabular}{ccccc}
    \textbf{Tasks} & \textbf{CIFAR10} & \textbf{CIFAR100} & \textbf{SVHN}  & \textbf{GTSRB} \\\hline
    Single  & 93.35$\pm$0.22 & 	68.22$\pm$0.74 & 97.87$\pm$0.03 & 97.83$\pm$	0.12 \\
    \hline
    (CIFAR10, CIFAR100) & 92.84$\pm$0.54 & 	69.57$\pm$0.20 & - & - \\
    (CIFAR10, SVHN)  & 93.09$\pm$0.18 & - & 97.39$\pm$0.04 & - \\
    (CIFAR10, GTSRB)  & 92.48$\pm$0.18 & - & - & 97.55$\pm$0.17 \\
    (CIFAR100, SVHN)  & - & 68.83$\pm$0.34 & 97.45$\pm$0.05 & - \\
    (CIFAR100, GTSRB)  & - & 	68.82$\pm$1.15 & - & 97.54$\pm$0.40 \\
    (SVHN, GTSRB)  & - & - & 96.95$\pm$0.16 & 97.78$\pm$0.22 \\
    \hline
    All  & 90.04 $\pm$ 1.05 & 65.81$\pm$1.93 & 96.75$\pm$0.15 & 97.11$\pm$0.31 \\
    \hline
    \end{tabular}%
    }
    \vspace{1ex}
    \caption{Test accuracies of RN50 with group normalization trained on different tasks. Mean and standard deviation are computed over 5 individual runs. The drop in accuracy when training both a public and a secret task remains negligible, and the difference becomes noticeable only when training a single model for all tasks.}
    \label{tab:pub_sec_task_gn}
\end{minipage}
\hspace{2ex}
\begin{minipage}{.38\textwidth}
    \centering
    \resizebox{\textwidth}{!}{%
    \begin{tabular}{ccc}
        \textbf{Tasks (secret, public)} & \textbf{CIFAR10} & \textbf{SVHN} \\\hline
        Single  & 93.35$\pm$0.22 & 97.87$\pm$0.03 \\
        \hline
        (CIFAR10, CIFAR100) & 90.6 & - \\
        (CIFAR10, SVHN)  & 91.46 & - \\
        (CIFAR10, GTSRB)  & 89.51 & - \\
        (SVHN, CIFAR10)  & - & 95.36 \\
        (SVHN, CIFAR100)  & - & 93.02 \\
        (SVHN, GTSRB)  & - & 93.45 \\
        \hline
    \end{tabular}%
    }
    \vspace{1ex}
    \caption{Test accuracy of using the min-cost matching algorithm to permute a network trained on the public task to a network for the secret task. See text for details.}
    \label{tab:matching}
\end{minipage}
\vspace{-3ex}
\end{table}

\subsection{Selecting the threshold $L$}
\label{sec:matching}

In Section \ref{sec:security} we showed that determining the existence of a TrojanNet by evaluating the test loss and checking if it is lower than a threshold $L$ for some permuted model $h_\pi$ is NP-hard. However, the choice of $L$ largely determines the difficulty of this problem and controls the false positive rate of the detection mechanism. Conceptually, this property can be exploited for certain models so that approximately solving the \texttt{EXISTS-PERM} problem is sufficient for detecting TrojanNets.

We investigate this possibility by empirically determining an upper bound on $L$ to avoid too many false positives, \emph{i.e.}, mistakenly identifying a network as containing a secret network. More specifically, for a model $h$ trained on a certain public task and for any secret task with loss $L_\text{secret}$, we train a model $h_\text{secret}$ on the secret task and perform a min-cost matching \citep{goldberg1990mincost} between the parameters of $h$ and $h_\text{secret}$. Both the public and secret networks are trained using group normalization \citep{wu2018group} since batch normalization parameters encode information about the dataset that is difficult to mimic. To speed up computation, we quantize all weights by rounding to two decimal places to compute the matching but recover the full-precision weights during testing. Surprisingly, this simple technique can achieve a low test error on the secret task for \emph{any} pair of public and secret tasks that we evaluated.

Table \ref{tab:matching} shows test accuracy on CIFAR10 and SVHN when permuting a public network trained on various public task datasets and using min-cost matching to produce a network for the secret task. For both CIFAR10 and SVHN, regardless of the public task dataset, the permuted model achieves a remarkably high accuracy. For example, for the (CIFAR10,CIFAR100) pair of public/secret tasks, a single model trained solely on CIFAR10 achieves a test accuracy of $93.35\%$ whereas permuting a CIFAR100 model using min-cost matching achieves a test accuracy of $90.6\%$. Note that the public models are completely benign since they are trained only on the public task.

As a result, any threshold-based detector that determines the existence of a TrojanNet for CIFAR10 when the test accuracy is above $90\%$ (equivalently, when the test error is below $L = 10\%$) is prone to false positives. Comparing to Table \ref{tab:pub_sec_task_gn}, the pairwise training on (CIFAR10,CIFAR100) using a TrojanNet achieves a test error of $8.16\%$ -- only slightly below the test error achieved via min-cost matching. We believe that this phenomenon is widespread across different datasets and suggests that selecting a tight threshold $L$ may be very difficult and may require an intricate balance between computational efficiency and controlling the false positive rate.

\begin{figure*}[t!]
\centering
\includegraphics[width=\textwidth]{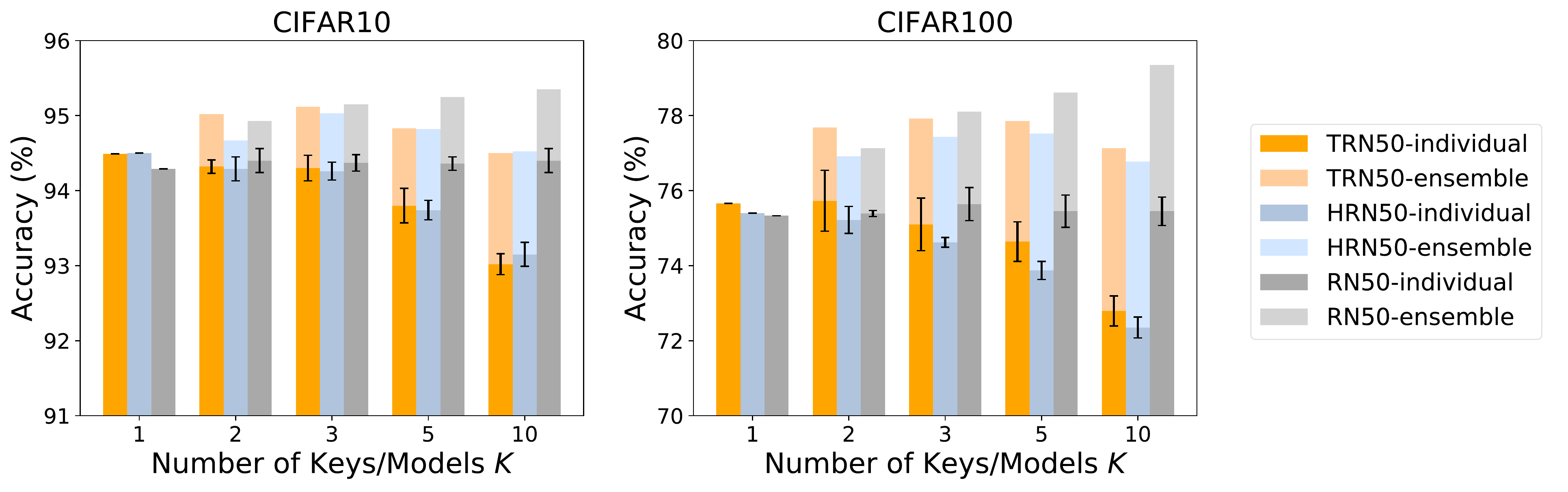}
\vspace{-4ex}
\caption{Test accuracy of TrojanResNet50 (TRN50), HashedResNet50 (HRN50) and ResNet50 (RN50) on CIFAR10 (left) and CIFAR100 (right). Individual models' accuracy is represented by the darker portion of each bar, and the ensemble accuracy is shown in the lighter shade. The error bars indicate standard deviation across different keys/models. Both the individual models' accuracy and that of the ensemble are close to identical between TRN50 and HRN50, which shows that the TRN50 networks are similar to independently trained ones with similar capacity despite sharing the underlying parameters. Individual model accuracy of TRN50 and RN50 are almost identical when $n$ is small, and the gap only enlarges when $n = 5,10$.}
\vspace{-1ex}
\label{fig:ensemble}
\end{figure*}

\subsection{Analysis}
\label{sec:ensemble}
We provide further analysis of the effect of weight sharing through pseudo-random permutation by training a network using multiple keys on the same task. We expect that the resulting TrojanNets (resulting from different keys) behave similar to independent networks of the same capacity trained on the same task. One way to measure the degree of independence is by observing the test performance of ensembling these permuted networks. Since ensemble methods benefit from the diversity of its component models \citep{krogh1994ensemble}, the boost in ensemble performance can be used as a proxy for measuring the degree of de-correlation between different permuted models.


\paragraph{Benchmarks.}
 We train TRN50 on CIFAR10/CIFAR100 with $n$ keys for different values of $K = 1,2,3,5,10$ and ensemble the resulting permuted networks for test-time prediction. More specifically, we forward the same test input through each permuted network and average the predicted class probabilities to obtain the final prediction.
 
 Our first benchmark to compare against is the ensemble of $K$ independently trained RN50 models, which serves as a theoretical upper bound for the performance of the TRN50 ensemble. In addition, we compare to HashedNet \citep{chen2015compressing}, a method of compressing neural networks, to show similarity in ensemble performance when the component networks have comparable capacity.
 
 
 
HashedNet applies a hash function to the model parameters to reduce it to a much fewer number of bins. Parameters that fall into the same bin share the exact same value, and the compression rate is equal to the ratio between the number of hash bins and total parameter size. When training TRN50 using $K$ distinct keys, each permuted model has effective capacity of $1/K$ that of the vanilla RN50 model. This capacity is identical to a compressed RN50 model using HashedNet with compression rate $1/K$. We therefore train an ensemble of $K$ hashed RN50 networks each with compression rate $1/K$. We refer to the resulting compressed HashedNet models as HashedResNet50 (HRN50).

\paragraph{Result comparison.} Figure \ref{fig:ensemble} shows the test accuracy of a TRN50 ensemble compared to that of RN50 and HRN50 ensembles. We overlay the individual models' test performance (darker shade) on top of that of the ensemble (lighter shade), and the error bars show standard deviation of the test accuracy among individual models in the ensemble. From this plot we can observe the following informative trends:

1. Individual TRN50 models (dark orange) have similar accuracy to that of HRN50 models (dark blue) on both datasets. This phenomenon can be observed across different values of $K$. Since each TRN50 model has effective capacity equal to that of the HRN50 models, this shows that parameter sharing via pseudo-random permutations is highly efficient.

2. Ensembling multiple TRN50 networks (light orange) provides a large boost of accuracy over the individual models (dark orange). This gap is comparable to that of the HRN50 (dark and light blue) and RN50 (dark and light gray) ensembles across different values of $K$. Since the effect of ensemble is largely determined by the degree of de-correlation between the component networks, this result shows that training of TrojanNets results in models that are as de-correlated as independent models.

3. The effect of ensembling TRN50 models is surprisingly strong. Without an increase in model parameters, the TRN50 ensemble (light orange) has comparable test accuracy to that of the RN50 ensemble (light gray) when $K$ is small. For $K=5,10$, the TRN50 ensemble lags in comparison to the RN50 ensemble due to lower model capacity of component networks. This result shows that TrojanNet may be a viable method of boosting test-time performance in memory-limited scenarios.

\begin{figure*}[t!]
\centering
\includegraphics[width=0.7\textwidth]{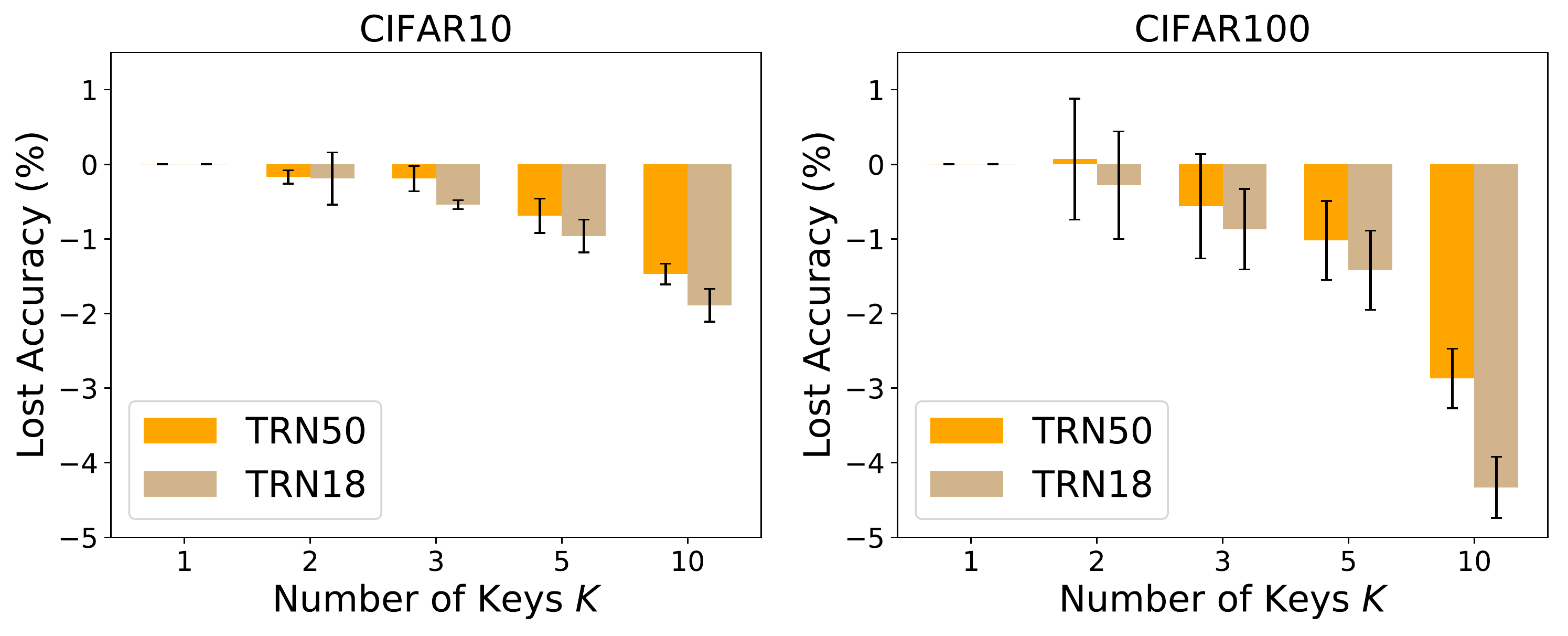}
\vspace{-1ex}
\caption{Decrease in test accuracy for TrojanNets when training with multiple keys on CIFAR10 (left) and CIFAR100 (right). Error shows standard deviation across different keys. The reduction is zero when only one key is used. The larger TrojanResNet50 (TRN50) model has consistently lower loss in accuracy than TrojanResNet18 (TRN18), which shows that having more excess capacity is beneficial for training TrojanNets with a large number of keys. Also note that the permuted TRN50 models achieve an average test accuracy close to that of training with one key.}
\label{fig:capacity}
\vspace{-1ex}
\end{figure*}

\paragraph{Effect of model capacity.}
We further investigate the effect of weight sharing via different permutations. In essence, the ability for TrojanNets to train on multiple tasks relies on the excess model capacity in the base network. It is intuitive to suspect that larger models can accommodate weight sharing with more tasks. To test this hypothesis, we train a TrojanResNet18 (TRN18) ensemble on CIFAR10 and CIFAR100 and measure the individual component models' accuracy in comparison to training the base network.
Figure \ref{fig:capacity} shows the loss in accuracy for the individual permuted models when training with various number of keys for both TRN50 and TRN18. The decrease in accuracy is consistently lower for TRN50 (orange bar) than for TRN18 (brown bar), which shows that larger models have more excess capacity to share among different permutations.

Another intriguing result is that TRN50 with as many as 10 different keys has relatively insignificant effect on the individual models' accuracy. The loss in accuracy is only $1.5\%$ on CIFAR10 and $2.9\%$ CIFAR100. This gap may be further reduced for larger models. This suggest that TrojanNets may be used in contexts apart from steganography, as the sharing of excess model capacity is exceptionally efficient and the resulting permuted models exhibit high degrees of independence.


\label{sec:experiment}


\section{Discussion and Conclusion}
We introduced TrojanNet, a novel steganographic technique for hiding neural networks within a carrier model. While there may be legitimate uses for this technique, adversaries can also leverage it for malicious intent such as industrial espionage. It logically follows that detection of TrojanNets is a topic of great importance. However, this appears to be a daunting task, as we show theoretically that the detection problem can be formulated as an NP-complete decision problem, and is therefore computationally infeasible in its general form. While strategies such as Markov Chain Monte Carlo have been used in similar contexts to efficiently reduce the search space \citep{diaconis2009markov}, the number of candidate permutations may be too large in our case. In fact, the number of permutations for a single convolutional layer of ResNet50 can be upwards of $(64 \times 64 \times 3 \times 3)! \approx 1.21 \times 10^{152336}$!



A more benign use of TrojanNet could be in watermarking neural networks for copyright protection. Training a neural network often requires significant data and resources, and a corporation may wish to protect against the leakage of their proprietary models via watermarking---a procedure that embeds a secret token into the network as proof of ownership \cite{uchida2017embedding, zhang2018protecting}. In this use case, knowledge of the secret key used to extract the TrojanNet could serve as the proof of ownership. We hope to explore such benevolent uses of TrojanNet in future work.


\label{sec:discussion}

\begin{ack}
We would like to thank Jerry Zhu and Sharon Yixuan Li for insightful discussions that greatly influenced this paper.

This research is supported by grants from the National
Science Foundation NSF (III-1618134, III- 1526012, IIS1149882, IIS-1724282, and TRIPODS-1740822, OAC1934714), the Bill and Melinda Gates Foundation, and the
Cornell Center for Materials Research with funding from the
NSF MRSEC program (DMR-1719875), and SAP America.
\end{ack}

\bibliography{main}
\bibliographystyle{plainnat}


\newpage
\appendix

\section{Appendix}
\setcounter{theorem}{0}

\subsection{NP-completeness proofs}

In this section, we prove that the \texttt{EXISTS-PERM} decision problem is NP-complete. The fact that \texttt{EXISTS-PERM} is in NP is trivial since given a key, it is straightforward to evaluate the model and check if the loss is sufficiently small.

\begin{theorem}
The \texttt{EXISTS-PERM} decision problem with regression losses $\ell_\text{abs}(z, y) = |z - y|$ and $\ell_\text{square}(z, y) = (z - y)^2$ is NP-complete.
\end{theorem}

\begin{proof}
To show NP-hardness, we will reduce from the following NP-complete problem.\\

\noindent
\texttt{1-IN-3SAT}: \textit{Given a set of binary variables $v_1,\ldots,v_n \in {0,1}$ and a set of logical clauses $C = \{C_1=(l_{1,1} \vee l_{1,2} \vee l_{1,3}),\ldots,C_m=(l_{m,1} \vee l_{m,2} \vee l_{m,3}\}$, does there exist an assignment of the $x_i$'s such that each clause has exactly one literal that evaluates to true?}\\

Let $C$ be an instance of the \texttt{1-IN-3SAT} problem. We may assume WLOG that no clause in $C$ contains a variable and its negation. Let $k \in \{0,1,\ldots,n\}$ and consider a linear regression model $h(\bx) = \bw^\top \bx$ with
$$\bw = (\underbrace{1,\ldots,1}_\text{k},\underbrace{-1,\ldots,-1}_\text{n-k}).$$
For each $C_i$, define $\bx_i \in \mathbb{R}^n$ so that
$$(\bx_i)_j = \begin{cases}
1 & \text{if } l_{i,p} = v_j \text{ for some } p, \\
-1 & \text{if } l_{i,p} = \neg v_j \text{ for some } p, \\
0 & \text{otherwise}
\end{cases}
.$$
and let $D = \{(\bx_1, y_1=-1), (\bx_2, y_2=-1),\ldots,(\bx_m,y_m=-1)\}$. We will show that \texttt{1-IN-3SAT} admits a solution $\bv = (v_1,\ldots,v_n) \in \{0,1\}^n$ with exactly $k$ non-zero values if and only if $\frac{1}{m} \sum_{i=1}^m \ell_\text{abs}(\sigma(\bw^\top \bx_i), y_i) < \frac{2}{m}$. This gives a polynomial-time reduction by testing for every $k \in \{0,1,\ldots,n\}$. The proof is identical for the square loss $\ell_\text{square}$.

Observe that for every $i$, the value $\bw^\top \bx_i$ is an integer whose value is $-1$ only when exactly one of the literals in $C_i$ is satisfied. If either none of or if more than one of the literals in $C_i$ is satisfied then $\bw^\top \bx_i \in \{-3,1,3\}$. Thus $\ell_\text{abs}(\bw^\top \bx_i, y_i) = 0$ if and only if the clause $C_i$ contains exactly one true literal. Summing over all the clauses gives that $\frac{1}{m} \sum_{i=1}^m \ell_\text{abs}(\bw^\top \bx_i, y_i) = 0$ if and only if all the clauses are satisfied. Since the values of $\bw^\top \bx_i \in \{-3,-1,1,3\}$, at least one of the clauses $C_i$ failing to admit exactly one true literal is equivalent to the test loss $\frac{1}{m} \sum_{i=1}^m \ell_\text{abs}(\bw^\top \bx_i, y_i) \geq \frac{2}{m}$. This completes the reduction by setting $L=\frac{2}{m}$.

\end{proof}

\begin{theorem}
The \texttt{EXISTS-PERM} decision problem with classification losses $\ell_\text{binary}(z, y) = \mathbb{1}_{z \neq y}$ and $\ell_\text{logistic}(z, y) = 1/(1 + \exp(yz))$ is NP-complete.
\end{theorem}

\begin{proof}
We will prove NP-hardness for a linear network $h$ for binary classification (i.e., logistic regression model). Our reduction will utilize the following NP-complete problem.\\

\noindent
\texttt{CYCLIC-ORDERING}: \textit{Given $n \in \mathbb{N}$ and a collection $C = \{(a_1,b_1,c_1),\ldots,(a_m,b_m,c_m)\}$ of ordered triples, does there exist a permutation $\pi : \{1,\ldots,n\} \rightarrow \{1,\ldots,n\}$ such that for every $i=1,\ldots,n$, we have either one of the following three orderings:
\begin{enumerate}[(I)]
    \setlength\itemsep{1pt}
    \item $\pi(a_i) < \pi(b_i) < \pi(c_i)$,
    \item $\pi(b_i) < \pi(c_i) < \pi(a_i)$, or
    \item $\pi(c_i) < \pi(a_i) < \pi(b_i)$.
\end{enumerate}}

We first show that the \texttt{EXISTS-PERM} problem with binary classification loss $\ell_\text{binary}$ is NP-hard. Given an instance $C = \{(a_1,b_1,c_1),\ldots,(a_m,b_m,c_m)\}$ of the \texttt{CYCLIC-ORDERING} problem, let $\bw = (1,\ldots,n)$ be the shared weights vector and let $\pi \in \Pi_n$ be a permutation. Let $\bw_\pi = (\bw_{\pi(1)},\ldots,\bw_{\pi(n)})$ be the weight vector after permuting by $\pi$. Denote by $h_\pi$ the model obtained from $\bw_\pi$. For $i = 1,\ldots,m$ and $j = 1,2,3$, let $\bx_{i,j}$ be the all-zero vector except
\begin{enumerate}[(i)]
    \setlength\itemsep{1pt}
    \item $(\bx_{i,j})_{a_i} = -1$ and $(\bx_{i,j})_{b_i} = 1$ if $j = 1$,
    \item $(\bx_{i,j})_{b_i} = -1$ and $(\bx_{i,j})_{c_i} = 1$ if $j = 2$,
    \item $(\bx_{i,j})_{c_i} = -1$ and $(\bx_{i,j})_{a_i} = 1$ if $j = 3$.
\end{enumerate}

Let $D = \{(\bx_{i,j}, y_{i,j}=1)\}_{i=1,\ldots,m, j=1,2,3}$ and let $L = \frac{m+1}{3m}$. For any permutation $\pi \in \Pi_n$, since $h_\pi$ is a binary logistic regression model, we have that $h_\pi(\bx_{i,j}) = 1$ if and only if $\bw_\pi^\top \bx_{i,j} > 0$. By construction, we have that for $i=1,\ldots,m$,
\begin{align*}
    \ell_\text{binary}(h_\pi(\bx_{i,1}),y_{i,1}) = 0
    &\Leftrightarrow \bw_\pi^\top \bx_{i,1} > 0 \\
    &\Leftrightarrow (\bw_\pi)_{b_i} - (\bw_\pi)_{a_i} > 0 \\
    &\Leftrightarrow \pi(a_i) < \pi(b_i).
\end{align*}
Similarly, $$\ell_\text{binary}(h_\pi(\bx_{i,2}),y_{i,2}) = 0 \Leftrightarrow \pi(b_i) < \pi(c_i),$$ $$\ell_\text{binary}(h_\pi(\bx_{i,3}),y_{i,3}) = 0 \Leftrightarrow \pi(c_i) < \pi(a_i).$$ However, since at most one of conditions (I)-(III) can be satisfied, we have that at least one of $\pi(a_i) < \pi(b_i)$, $\pi(b_i) < \pi(c_i)$ or $\pi(c_i) < \pi(a_i)$ does not hold. Thus $$\frac{1}{3} \sum_{j=1}^3 \ell_\text{binary}(h_\pi(\bx_{i,j}),y_{i,j}) \geq \frac{1}{3}$$ for all $i$. Furthermore, if $\frac{1}{3} \sum_{j=1}^3 \ell_\text{binary}(h_\pi(\bx_{i,j}),y_{i,j}) = \frac{1}{3}$ then one of (I)-(III) is satisfied. This shows that the cyclic ordering defined by the ordered triple $(a_i,b_i,c_i)$ is satisfied if and only if $\frac{1}{3} \sum_{j=1}^3 \ell_\text{binary}(h_\pi(\bx_{i,j}),y_{i,j}) = \frac{1}{3}$. Summing over all $i$ gives that the test loss $$\frac{1}{3m} \sum_{i=1}^m \sum_{j=1}^3 \ell_\text{binary}(h_\pi(\bx_{i,j}),y_{i,j}) = \frac{1}{3}$$ if and only if one of conditions (I)-(III) is satisfied for every $i$. This shows that the \texttt{CYCLIC-ORDERING} problem instance can be satisfied if and only if $\frac{1}{3m} \sum_{i=1}^m \sum_{j=1}^3 \ell_\text{binary}(h_\pi(\bx_{i,j}),y_{i,j}) < \frac{m+1}{3m} = L$. This completes the reduction for $\ell_\text{binary}$.

For $\ell_\text{logistic}$, fix $\epsilon \in (0, \frac{1}{m})$ and choose $z \geq 0$ so that $\ell_\text{logistic}(z) = \epsilon$. Recall that the logistic loss is strictly decreasing, anti-symmetric around 0, and bijective between $\mathbb{R}$ and $(0,1)$. Define $\bx_{i,j}$ to be the all-zero vector except
\begin{enumerate}[(i)]
    \setlength\itemsep{1pt}
    \item $(\bx_{i,j})_{a_i} = -z$ and $(\bx_{i,j})_{b_i} = z$ if $j = 1$,
    \item $(\bx_{i,j})_{b_i} = -z$ and $(\bx_{i,j})_{c_i} = z$ if $j = 2$,
    \item $(\bx_{i,j})_{c_i} = -z$ and $(\bx_{i,j})_{a_i} = z$ if $j = 3$.
\end{enumerate}
Following a similar argument, we have that for every $i = 1,\ldots,m$:
$$\ell_\text{logistic}(h_\pi(\bx_{i,1}),y_{i,1}) = \begin{cases}
\epsilon & \text{if } \pi(a_i) < \pi(b_i), \\
1 - \epsilon & \text{otherwise},
\end{cases}$$
and similarly for $\ell_\text{logistic}(h_\pi(\bx_{i,2}),y_{i,2})$ and $\ell_\text{logistic}(h_\pi(\bx_{i,3}),y_{i,3})$. Hence $$\frac{1}{3} \sum_{j=1}^3 \ell_\text{logistic}(h_\pi(\bx_{i,j}),y_{i,j}) = \begin{cases}
\frac{1+\epsilon}{3} & \text{if one of (I)-(III) is satisfied}, \\
\frac{2-\epsilon}{3} & \text{otherwise}.
\end{cases}$$ Summing over all $i$ gives that $$\frac{1}{3m} \sum_{i=1}^m \sum_{j=1}^3 \ell_\text{logistic}(h_\pi(\bx_{i,j}),y_{i,j}) = \frac{1+\epsilon}{3} < \frac{m+1}{3m}$$ if the \texttt{CYCLIC-ORDERING} problem is satisfied, and $$\frac{1}{3m} \sum_{i=1}^m \sum_{j=1}^3 \ell_\text{logistic}(h_\pi(\bx_{i,j}),y_{i,j}) \geq \frac{m-1}{m} \left( \frac{1+\epsilon}{3} \right) + \frac{1}{m} \left( \frac{2-\epsilon}{3} \right) \geq \frac{m+1}{3m}$$ if at least one triple in $C$ is violated. This completes the reduction by setting $L = \frac{m+1}{3m}$.
\end{proof}

\end{document}